\def\arxiv{}
\documentclass[conference]{IEEEtran}
\usepackage{times}
\usepackage[numbers,sort]{natbib}
\usepackage{multicol}
\usepackage[bookmarks=true,hidelinks]{hyperref}
\usepackage{graphicx}
\usepackage{amsmath}
\usepackage{amsfonts}
\usepackage{amsthm}
\usepackage{mathtools}
\usepackage{accents}
\usepackage{color}
\usepackage{titling} 
\usepackage{multicol} 
\usepackage{ifthen}

\theoremstyle{plain}
\newtheorem{definition}{Definition}
\newtheorem{theorem}{Theorem}

\newtheorem{lemma}{Lemma}

\theoremstyle{definition}
\newtheorem{example}{Example}
\newtheorem{construction}{Construction}

\newtheoremstyle{simple}
{}                
{}                
{}        
{}                
{}       
{\hspace*{-4pt}}               
{ }               
{}                
\theoremstyle{simple}

\newenvironment{proofsansqed}{\noindent\paragraph*{{Proof roadmap}}{}}

\newenvironment{proofdeferred}{\noindent\paragraph*{{Proof}}{}}

\definecolor{markercolor}{rgb}{.19, .19, .48}

\newcounter{tecounter}
\newenvironment{tightenumerate}
{
    \begin{list}{\arabic{tecounter}\addtocounter{tecounter}{1})}{%
    \setcounter{tecounter}{1}
        \setlength{\leftmargin}{14pt}
        \setlength{\topsep}{1pt}
        \setlength{\partopsep}{0pt}
        \setlength{\itemsep}{2pt}
        \setlength\labelwidth{10pt}%
        }
        \ignorespaces}
{\unskip\end{list}}

\newenvironment{tightitemize2}
{
    \begin{list}{$\bullet$}{%
        \setlength{\leftmargin}{18pt}
        \setlength{\topsep}{3pt}
        \setlength{\itemsep}{4pt}%
        \setlength\labelwidth{12pt}%
        }
        \ignorespaces}
{\unskip\end{list}}

\def\hat{\widehat}

\newlength{\dhatheight}
\newcommand{\doublehat}[1]{%
\settoheight{\dhatheight}{\ensuremath{\hat{#1}}}%
\addtolength{\dhatheight}{-0.35ex}%
\hat{\vphantom{\rule{1pt}{\dhatheight}}%
\smash{\hat{#1}}}}

\newcommand{\compcent}[1]{\vcenter{\hbox{$#1\circ$}}}

\newcommand{\comp}{\mathbin{\mathchoice
  {\compcent\scriptstyle}{\compcent\scriptstyle}
  {\compcent\scriptscriptstyle}{\compcent\scriptscriptstyle}}}

\newcommand{\rind}[1]{^{(#1)}}  
\newcommand{\rindi}{\rind{i}} 
\newcommand{\rindj}{\rind{j}} 
\newcommand{\rindn}{\rind{n}}

\newcommand{\Z}{\mathbb{Z}}
\newcommand{\R}{\mathbb{R}}
\newcommand{\plus}{^{+}}
\newcommand{\vin}{\rotatebox[origin=c]{-90}{\footnotesize$\in$}}  
\newcommand{\Vin}{\rotatebox[origin=c]{90}{{\footnotesize$\in$}}}

\newcommand{\seen}[1]{\ensuremath{#1^{\text{occ}}}} 

\newcommand{\foreign}[1]{\emph{#1}} 

\newcommand{\vmin}{v_{\rm min}}
\newcommand{\vmax}{v_{\rm max}}
\newcommand{\vhatmin}{\hat{v}_{\rm min}}
\newcommand{\vhatmax}{\hat{v}_{\rm max}}

\newcommand{\disks}{_{\text{disks}}}
\newcommand{\single}{_{\text{single}}}

\newcommand{\gobble}[1]{}
\newcommand{\gobblesubs}[2]{{#2}}

\date{\vspace*{-2\baselineskip}}

\begin{document}

\title{Reality as a simulation of reality: robot illusions,\\fundamental limits, and a physical demonstration}

\IEEEoverridecommandlockouts

\author{Dylan A. Shell and Jason M. O'Kane\thanks{D. A. Shell is with the
Department of Computer Science and Engineering at Texas A\&M University.  J. M.
O'Kane is with the Department of Computer Science and Engineering at the
University of South Carolina.  This material is based upon work supported by
the National Science Foundation under Grant Nos.  1526862, 1527436, 1849249,
and 1849291.}}

\maketitle

\begin{abstract}
We consider problems in which robots conspire to present a view of the world
that differs from reality.  The inquiry is motivated by the problem of
validating robot behavior \emph{physically} despite there being a discrepancy
between the robots we have at hand and those we wish to study, or the
environment for testing that is available versus that which is desired, or
other potential mismatches in this vein.  After formulating the concept of a
convincing illusion, essentially a notion of system simulation that takes place
in the real world, we examine the implications of this type of simulability in
terms of infrastructure requirements. Time is one important resource: some
robots may be able to simulate some others but, perhaps, only at a rate that is
slower than real-time.  This difference gives a way of relating the simulating
and the simulated systems in a form that is relative. We establish some
theorems, including one with the flavor of an impossibility result, and
providing several examples throughout.  Finally, we present data from a simple
multi-robot experiment based on this theory, with a robot navigating amid an
unbounded field of obstacles.
\end{abstract}

\bigskip
{\footnotesize
\noindent``Truth is beautiful, without doubt; but so are lies.''---\,Ralph Waldo Emerson

}

\IEEEpeerreviewmaketitle

\section{Motivation and Overview}


Robotics papers usually include evidence of algorithms or controllers that have
been executed or evaluated on some kind of system, typically comprising either
physical robots or a substitute.  But what constitutes a robot demonstration,
exactly?  One division is generally drawn between software simulation and real
robots.  This is, at best, a rather rough distinction for there is a spectrum
of simulators spanning a wide range of fidelities. Actually, the same might be
said for physical robots: data and conclusions from robots using, say, a
sophisticated external motion capture system, or cameras with visual markers,
might be representative of robots operating in the field with GPS. Or, on the
other hand, depending on what you're trying to do, they might not.

What is certain is that there are more choices, between full software
simulation and full physical implementation, than are generally recognized or
garner attention (see Figure~\ref{fig:phil}).  Inasmuch as this is critical for
robotics as a scientific enterprise, it is perhaps curious that there has been
little formal treatment of representativeness or verisimilitude beyond the
complete hardware and software extremes, and their consideration.  This paper's
\foreign{raison d'\^etre} is to initiate a close, systematic examination of
these other options.

\begin{figure}
\includegraphics[width=\linewidth]{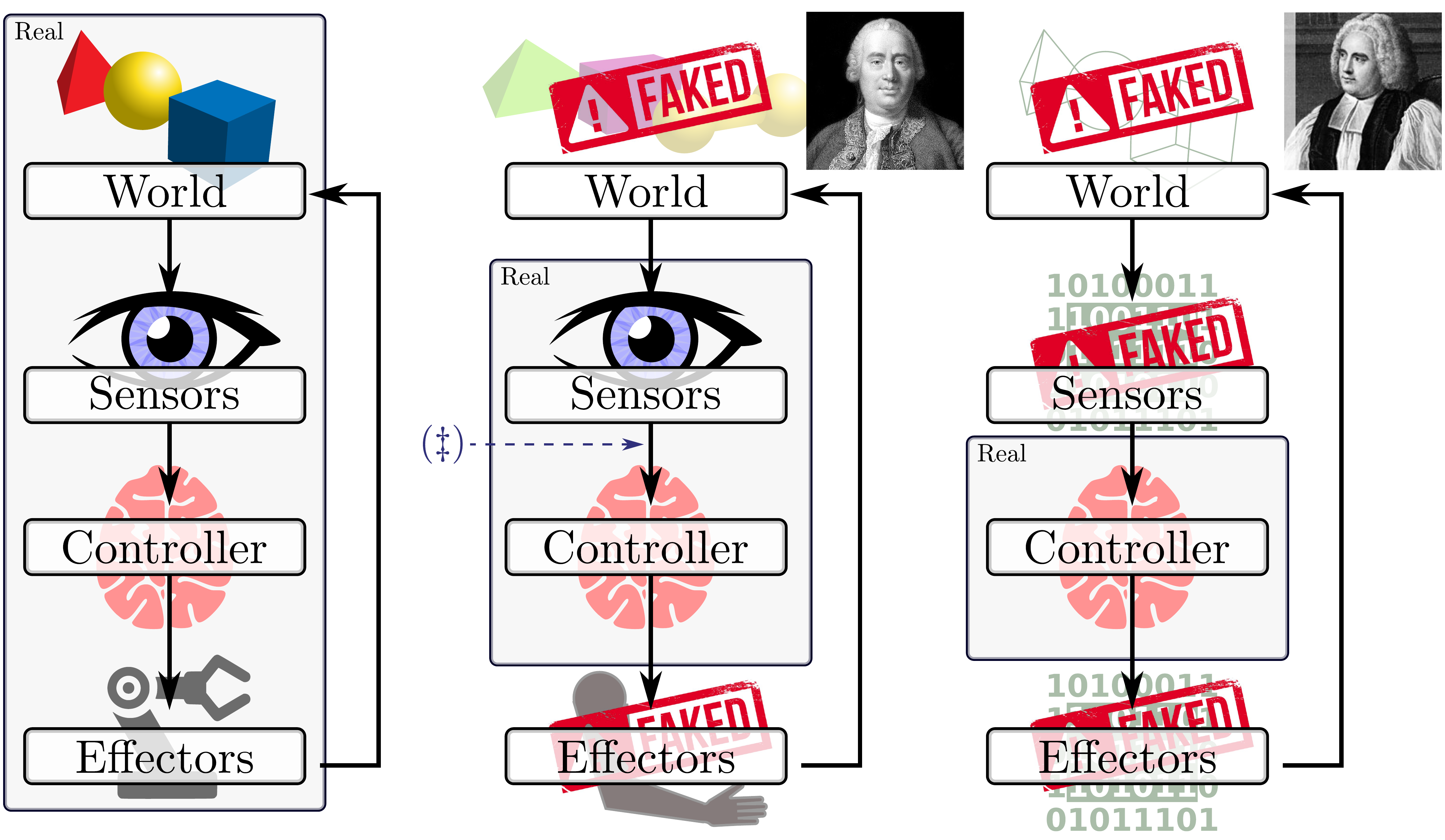}
\vspace*{-14pt}
\caption{An illustration of different modes of fakery matched with appropriate
intellectual positions: the orthodox view (left) and two forms of philosophical
skepticism (center and right).  The left column represents an end-to-end
physical robot experiment with every part being real.  
The middle column, bearing the bust of David Hume,
involves real sensors but also reason to doubt the veracity of the world
they report. 
The rightmost column, headed by Bishop Berkeley, has a real controller but all
other elements are mere software---it represents a robot experiment conducted
in simulation (for the commonplace use of that term) where information is
injected directly into the robot's cortex.
Hume's column is the case studied in the paper: robots perceive a
world mediated by sensors that are grounded in the real world, but
it is a world of ``nothing but sophistry and illusion.''
\label{fig:phil} } 
\vspace*{-14pt}
\end{figure}

We want to understand how one physical system may be used to mimic the behavior
of another.  By system, we are considering a setting where observations are
made (via sensors) and used to choose actions that are effected (via
actuators) and this unfolds over time.  We begin with a simplified
discrete-time setting (Definition~\ref{defn:sys}) where we can contemplate exact emulation (Definition~\ref{defn:illusion}), rather than
considering approximate or imprecise imitation.
The central features which distinguish the approach from other formalisms of
emulation between robot systems (see Section~\ref{sec:related}) are the possibility of variable time expansion (somewhat akin to Milner's weak bisimulation~\cite{milner1982four}) and a narrow focus on mimickry only up to the perceptual capabilities of the system under emulation.

We then formulate some
particular questions, such as: 
``What are the resources involved, how do we quantify resource requirements,
and relate them?'' (Definition~\ref{defn:slowdown},
Theorem~\ref{thm:noboundedillusion}), 
``How do we compose or nest such systems?'' (Theorem~\ref{thm:comp}), 
``What happens to these things when systems are modified
(Theorem~\ref{thm:coarser_obs})'', etc.
Even this simple setting is replete with possibilities, some of which are
both exciting and enticing (Section~\ref{sec:poss_ideas}). 

In terms of immediate utility for the practitioner, the present paper shows how
to conduct a novel sort of emulation with real hardware where sensors,
rather than being faked out of whole cloth---as is usually done with computational or mathematical
models that are highly idealized coarse approximations---provide real signals.
As the instances we study herein show, there may be considerable freedom in
choosing different ways to emulate one system with another, with implications
for future robotic laboratory infrastructure.

\section{Related work}\label{sec:related}

\gobblesubs{%
Part of the inspiration for the present work derives from recent studies of animal
biomechanics, including their behavior when their `sensor inputs' are altered
within a feedback loop. We first describe these, then turn to relevant work
closer to home in the robotics community.%
}{%
}

\gobblesubs{%
\subsection{Animal studies}
}{%
\subsection{Animal studies: The inspiration for the present work}
}


For decades,  biologists have sought to chart the perceptual limits
of organisms and to understand how informational mismatches affect behavior~\cite{waterman89animal,dusenbery92sensory}. Recent years have seen virtual and augmented reality technologies
being used in this quest with considerable enthusiasm.  The animals studied range from small
mammals~\citep{holscher05} down to insects~\cite{taylor08}, being studied both while
walking~\citep{takalo12} and
flying~\citep{fry04,fry08,luu11}.
The journal \emph{Current Zoology} recently devoted a special issue to the
topic~\citep{witte17}. 
As a concrete example, Takalo et al. constructed a
laboratory apparatus comprising a spherical
projection surface and a track ball that enables
the detailed study of the walking behavior of the
cockroach {\em (periplaneta americana)} 
\gobblesubs{%
by giving
it synthetic visual stimuli~\cite{takalo12}.
These sorts of technologies, refashioned as scientific apparatus, have the
advantage of unprecedented flexibility: they permit the study of an organism
with precise regulation of the environs, even in naturalistic settings,
enabling a systems-level understanding of the organism.
}{%
by providing it with synthetic visual stimuli,
ultimately to give a systems-level understanding of 
the organism~\cite{takalo12}.
}


\subsection{Practical simulation in software}

Software simulations are an inescapable part of the current robotics research
landscape, 
with the community devoting much time and attention to related questions,
including through the biennial {\sc simpar} conference. %
The software traces out some element of a robot's execution in a
virtual, rather than physical world, generating artificial sensor readings 
(or sometimes \gobble{other higher-level} state information), and evolving the robot system
forward in time.

\gobblesubs{
At the center of most discussions about purely software robot simulation is the
concept of fidelity: How closely does the simulator mimic the real world?}{
Center to most discussions about software simulation are 
considerations of fidelity: How closely does the simulator mimic the real
world?  }
High-fidelity simulation software like\gobblesubs{, for example,}{}
Gazebo~\cite{koenig2004design} has been developed to account for many of the
complications experienced by the \gobble{types of}complex robots native to many
research labs.
But fidelity may be traded for other features as
\gobblesubs{, for example, in the Stage simulator which aims to}%
{some efforts} strike a ``useful balance between fidelity and
abstraction''~\cite{vaughan2008massively}.
Other simulators, designed for specific robot
types~\cite{jakobi1997evolutionary,craighead2007survey,pinciroli2011argos,pekarek2011backwards,curet2010versatile,stanley1992computer},
optimization/control schemes~\cite{todor12mujoco}, and application
domains~\cite{carpin2007usarsim,mondada2004swarm,rodriguez2010framework}, \gobblesubs{have
also been developed}{exist}.

\gobblesubs{
The present work may be viewed in part as a generalization of this traditional
notion of robot simulation, in which certain elements of the simulation are
conducted physically, rather than virtually.
This direction is closely related to work on robot simulations, visualizations, and development
environments that alter aspects of the physical world using mixed or
augmented reality
techniques~\cite{amstutz02,stilman2005augmented,collett06,chen09}.  The
distinguishing feature of the present work is that modification of things
in the real world is done by and for the benefit of robots, not for 
human developers or operators, and not necessarily by additional technologies
(like visual projectors or displays).
}{
This work is partly a generalization of the traditional
notion of robot simulation, but with elements of the simulation 
conducted physically rather than virtually.
Closely related work includes endeavors that alter aspects of the physical
world using mixed or augmented reality
techniques~\cite{amstutz02,stilman2005augmented,collett06,chen09}.  
The
distinguishing here is that modifications of 
the world are made by robots and for robots, not 
human developers or operators, and not via additional display technologies.
}

\subsection{Simulation as theoretical concept}

\gobblesubs{
Relating two systems by the fact that one can simulate the other, for some
definition of simulation, is an oft-recurring theoretical standpoint.  Among
other instances, the symmetric notion, where two systems are each able to
match the other, yields the concept of bisimulation, which is an equivalence
relation.  The bisimilarity idea was identified independently in modeling
concurrent systems~\cite{Milner1980} and in modal logic~\cite{Park1981}; it
may also be understood via game theoretic definitions~\cite{stirling95local}.
}{
Relating systems by the fact that they can simulate each other, for some
definition of simulation, is
a recurring theoretical theme.
The symmetric notion, where two systems are each able to
match the other, yields the concept of bisimulation, which is an equivalence
relation.  Bisimilarity was identified independently in modeling
concurrent systems~\cite{Milner1980} and in modal logic~\cite{Park1981}; it
also has a game theoretic interpretation~\cite{stirling95local}.
}

\gobblesubs{
With specific regard to theory focused on robots, the invariants among
sensori-computation circuits of \citet{donald95information} are shown by
comparing two circuits systems under a suitably chosen definition of simulation
(that includes calibration information, for instance).
The construction of a lattice relating and identifying the relative power of
robots was shown by \citet{OKaLav08}, wherein they introduce a simulation concept
that connects sensors and actuators, and policies between, permitting, for instance,
a robot with a bump sensor to physically cast rays in the environment in order to
establish information comparable to a laser range finder. One observes that doing
so might lead one to reason about efficiency considerations, say in terms of
time or energy---ideas that are similar to the work explored herein.
Indeed, the dominance relation between robot systems introduced by O'Kane and
LaValle, though presented only for single robots, bears a strong parallels to
the notion of illusion we introduce here, particularly in the use of one system
to emulate certain properties of another, possibly with the aid of differences
in the systems' time scales.  Our illusions differ because they directly and
necessarily consider presence of multiple robots in the system, and because
they focus on a notion of perceptual equivalence for the robots participating
in the illusion, rather than on the overall scope of what can be inferred by
robots in each system under certain information maps.
}
{
Closer to home in robotics, 
invariants among sensori-computation circuits of \citet{donald95information},
and the dominance relation between robot systems introduced by \citet{OKaLav08},
bear parallels to the notion of illusion we introduce here, particularly in the
use of one system, or re-arrangements of the resources contained therein, to
emulate certain properties of another. In this paper, the emphasis on
perceptual equivalence for the robots participating in the illusion is fresh.
}


\section{Preliminary definitions}

\subsection{Systems}
We wish to talk about relationships between pairs of systems of robots. First,
then, we need to define the notion of a system. Because henceforward we shall
consider systems consisting of possibly many robots, \gobble{and it'll be helpful to 
think in a roughly global way,} we jump directly into definitions that consider
(potentially) multiple robots.
\gobblesubs{Through all of the notation that follows, we
 use superscripts in parentheses to
denote robot indices.  Likewise subscripts denote time indices.
}{Superscripts in parentheses denote robot indices; subscripts are time indices.}

\begin{definition}\label{defn:sys}
  A \emph{deterministic multi-robot transition system} is a $7$-tuple
    $(n, X, U, f, Y, h, x_0)$,
  in which
  \begin{enumerate}
    \item $n$ is a positive integer identifying the number of robots,
    \item $X = X\rind{1} \times \cdots \times X\rindn$ denotes a state
    space, composed of individual state spaces for each robot,
    \item $U = U\rind{1} \times \cdots \times U\rindn$ denotes an action
    space, composed of individual action spaces for each robot,
    \item $f: X \times U \to X$ is a state transition function, defined in
    terms of transition functions $f\rind1,\ldots,f\rindn$ for each robot, so
    that\vspace*{-4pt}
    {\small
    \begin{align*} 
      f & \left( (x\rind1, \ldots, x\rindn), (u\rind1,\ldots,u\rindn)\right) \\[-2pt]
        & = \left( f\rind1(x, u\rind1),\ldots, f\rindn(x, u\rindn)  \right).
    \end{align*}\vspace*{-10pt}}
    \item $Y = Y\rind{1} \times \cdots \times Y\rindn$ denotes an observation
    space, composed of individual observation spaces\gobble{for each robot},
    \item $h: X \to Y$ is an observation function, defined in
    terms of observation functions $h\rind1,\ldots,h\rindn$ for each robot, so
    that
      $h \left( x \right)
        = \left( h\rind1(x), \ldots, h\rindn(x)  \right)
      $,
    \item $x_0 \in X$ is the system's initial state.

  \end{enumerate}
\end{definition}
\vspace*{-12pt}

\bigskip
\noindent 
Such a system evolves, in a series of discrete time steps, subject to the
following pair of equations:
\vspace*{-2pt}
\begin{align*} 
X\hspace*{11pt}  &\hspace*{22pt} X \hspace*{9pt} U \\[-8pt]
\Vin\hspace*{13pt}  &\hspace*{24pt} \Vin \hspace*{11pt}\Vin \\[-8pt]
 x_{k+1} &= f(x_{k}, u_{k}),\\
 y_{k} &= h(x_{k}).\\[-6pt]
\vin\hspace*{2pt}  &\hspace*{24pt}\vin\\[-6pt]
Y\hspace*{1pt}  &\hspace*{22pt}X
\end{align*} 
\vspace*{-12pt}

\begin{figure}
  \centering
  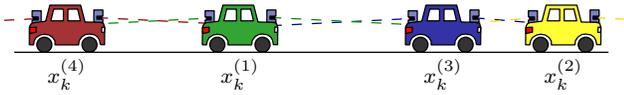
  \vspace*{-4pt}
  \caption{An example of the sort of system in
  Example~\ref{ex:caravan}, with $n=4$.
  At each time step $k$, each robot is at some point $x\rindi_k$ along the roadway moving with velocity $u\rindi_k$, and measures the distances $y\rindi_k$ to the adjacent robots.
  }
  \label{fig:caravan}
  \vspace*{-12pt}
\end{figure}

\gobblesubs{\noindent A few simple examples, each of which we revisit later, can
illustrate the idea.}{\noindent A few simple examples, to be revisited later, 
illustrate the idea.}
\begin{example}\label{ex:caravan}
  Consider a caravan of $n$ autonomous vehicles\,---that is, robots---\,moving
  down a long single-lane roadway.
  Suppose each robot can control its own velocity, subject to some upper and
  lower bounds, and can also measure the distance to the other robots
  immediately in front of and immediately behind itself.  See
  Figure~\ref{fig:caravan}.
  We might describe this scenario as a deterministic multi-robot transition system
  \vspace*{-2pt}
    \begin{equation}
      S_{n,\vmin,\vmax} = (n, \R^n, [v_{\rm min}, v_{\rm max}]^n, f, \R\plus \times \R\plus, h, x_0), \label{eq:av}
  \vspace*{-2pt}
    \end{equation}
  for which we'll give the state transition function $f$ and observation $h$
  shortly.  Here elements of the state space $X=\R^n$ encode the position,
  along the one-dimensional roadway, of each of the $n$ robots.
  At each time step $k$, the action $u\rindi_k \in [\vmin, \vmax]$ of
  robot $i$ denotes the velocity of that robot at that time.  Thus, we may
  define $f(x, u) = x + u$.  We assume that $\vmin < \vmax$.
  Each observation $y\rindi_k \in \R\plus\times\R\plus$ is a pair of
  integers indicating the distance to the closest other robot, if any, in
  each direction:
  \vspace*{-4pt}
    \begin{align*}
      y\rindi_k \!=\! h\rindi(x_k) \!=\! \bigg(\!\!
        &\min \left(\! \left\{ |x\rindj_k - x\rindi_k| \middle| x\rindj_k < x\rindi_k \right\} \cup \{ \infty \} \right) ,\\[-4pt]
        &\min \left(\! \left\{ |x\rindj_k - x\rindi_k| \middle| x\rindj_k > x\rindi_k \right\} \cup \{ \infty \} \right)
      \!\!\bigg).
  \vspace*{-2pt}
    \end{align*}
  To refer to the individual measurements in a single observation, we use the
  symbols $b$ and $a$ for the distances \underline{b}ehind and
  \underline{a}head, so that $y\rindi_k = (b\rindi_k, a\rindi_k)$.
  Finally, the initial state $x_0$ is some known but arbitrary state.

  Notice that \eqref{eq:av} is, in fact, defining an infinite family
  of systems, parameterized by the number of vehicles in the system and the
  ranges of allowable velocities.

  This is, of course, a heavily idealized model of caravaning autonomous
  vehicles, crafted as an elementary illustration of Definition~\ref{defn:sys}.
  Richer models might, for example, expand $X$ to model multi-lane roadways or
  the robots' lateral positions within the lanes, enrich $U$ and $f$ to model
  the dynamics of some physical system more faithfully, or modify $Y$ and $h$
  to model, say, a \textsc{lidar} sensor with greater fidelity.
\end{example}

\begin{example}\label{ex:robotarium}
  Consider a system in which many small disk-shaped differential drive robots
  move in a shared, bounded, planar workspace, with each robot aware of the
  relative positions of the other robots within some small sensor range.
  Refer to Figure~\ref{fig:pillars}[left].
  One might realize this kind of system using, for example,
  Khepera~\cite{mondada1999development}, r-one~\cite{mclurkin2013using}, or
  GRITSbot~\cite{pickem2015gritsbot} robots.
  We can model such a system by choosing the number of robots $n$, the
  rectangular workspace $W \subseteq \R^2$, the maximum wheel velocity
  $\vmax$, and the sensor range $r$.  We then define
  \vspace*{-4pt}
  \begin{equation}
    S\disks = \left(n, X\disks, U\disks, f\disks, Y\disks, h\disks, x_0 \right),
  \vspace*{-4pt}
  \end{equation}
  in which
    the states in $X_{\rm disks} = (W \times S^1)^n$,
    the actions in $U_{\rm disks} = [-\vmax,\vmax]^2$ denote the left
    and right wheel velocities for each robots,
    the state transition function $f\disks$ encodes the well-known
    kinematics for differential drive robots,
    the observations in $Y\disks = \cup_{i=0}^{n-1} (\R^2)^n$ are lists
    of between $0$ and $n-1$ planar positions,
    the observation function $h\rindi\disks$ for each robot $i$ returns
    a list of the relative positions of any other robots within distance $r$ of robot $i$, and
    the initial state $x_0 \in X\disks$ is a known but arbitrary state.
\end{example}

\begin{example}\label{ex:pillars}
  Definition~\ref{defn:sys} is also suitable for describing single-robot
  systems as a particular case with $n=1$.
  For example, a velocity-controlled robot moving in a very large field of
  nearly-identical static obstacles, with a sensor to detect those obstacles
  when they are nearby, might be modelled as
  \vspace*{-4pt}
  \begin{equation}
    S\single = \left(1, X\single, U\single, f\single, Y\single, h\single, x_0 \right),
  \vspace*{-2pt}
  \end{equation}
  with
    $X\single = \R^2$, $U\single = [-\vmax, \vmax]$, and $f\single(x, u) = x+u$.
  The observation space $Y\single$ and $h\single$ may be defined to return the
  locations of the center points of each obstacle.  See
  Figure~\ref{fig:pillars}[right].
\end{example}

\begin{figure}
  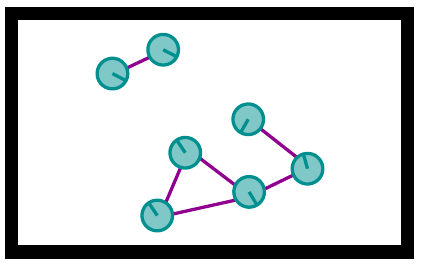
  \hfill
  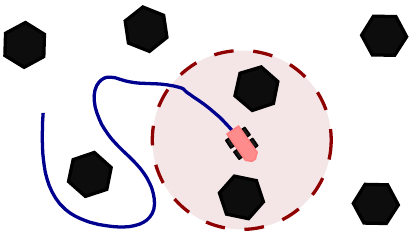
  \caption{[left] A team of simple robots in a bounded environment, as in
  Example~\ref{ex:robotarium}.  [right] A single robot moving in a unbounded
  field of obstacles, as in Example~\ref{ex:pillars}.}
  \vspace*{-16pt}
  \label{fig:pillars}
\end{figure}

\vspace*{-4pt}
\subsection{Policies}
In the model, a robot operates by choosing actions to execute, a concept
detailed via a policy.  The essential\gobble{ and important} question in formalizing
policies is to determine what information is used by the robot in considering its
action.  Now, to define the policy concept, we adopt the style of LaValle's book~\cite{Lav06}.

We begin, first, with something 
simple that will turn out to be inadequate for our needs.
If robot $i$, at time step $k$, has sufficient information that it can determine
its state, i.e., it is a fully observable problem, then its policy $\pi\rindi$
might be defined as a function of that state:
\vspace*{-4pt}
\begin{align} 
X\rindi\hspace*{-5pt}  &\hspace*{26pt} U\rindi \nonumber\\[-6pt]
\Vin\hspace*{7pt}  &\hspace*{28pt} \nonumber\Vin\\[-4pt]
x\rindi_{k} & \xmapsto{\pi\rindi} u\rindi_{k}.\label{eq:simplest_policy}
\vspace*{-4pt}
\end{align} 
More likely, the robot will only have access to its history of actions and
observations to select its action
\vspace*{-4pt}
\begin{align} 
U\rindi\hspace*{18pt} U\rindi\hspace*{6pt} Y\rindi\hspace*{16pt} Y\rindi\hspace*{-3pt}  &\hspace*{26pt} U\rindi \nonumber\\[-6pt]
\Vin\hspace*{31pt} \Vin\hspace*{17pt} \Vin\hspace*{30pt} \Vin\hspace*{7pt}  &\hspace*{28pt} \Vin\nonumber\\[-4pt]
u\rindi_{0}, \dots,u\rindi_{k},y\rindi_{0}, \dots,y\rindi_{k} & \xmapsto{\pi\rindi} u\rindi_{k}. \label{eq:own_policy}
\end{align} 
In what follows, one robot system will seek to present some
view of the world to match a description as will be seen by some other, secondary system.
This primary system must know some aspects of that other system to fool it effectively.
That is, the primary system must be aware of the `fourth wall' and know some of the
expectations and qualities on the other side of it.
Throughout, we use a notational convention: 
we distinguish the primary system (initially best thought of as the physical
system) by placing a hat over its variables;
all variables for the secondary
are bare. %
Now, returning to our formalization of the policy concept, we must
generalize the notation so far in order for it to present information about the
primary system and a secondary one, partitioned like such
\begin{align}
\hat{U}\rindi\hspace*{18pt} \hat{U}\rindi\hspace*{6pt} \hat{Y}\rindi\hspace*{18pt} \hat{Y}\rindi\hspace*{3pt} X\hspace*{22pt} X\hspace*{2pt}  &\hspace*{24pt} \hat{U}\rindi\nonumber\\[-6pt]
\Vin\hspace*{31pt} \Vin\hspace*{17pt} \Vin\hspace*{30pt} \Vin\hspace*{15pt} \Vin\hspace*{26pt} \Vin\hspace*{5pt}  &\hspace*{28pt} \Vin\nonumber\\[-4pt]
\hat{u}\rindi_{0}, \dots,\hat{u}\rindi_{k},\hat{y}\rindi_{0}, \dots,\hat{y}\rindi_{k} ,x_{0}, \dots,x_{\ell} & \xmapsto{\hat{\pi}\rindi} \hat{u}\rindi_{k}.\label{eq:full_policy}\\[-10pt]
\underset{\substack{i\text{'s action} \\[1.5pt] \text{history}}}{\underbrace{\hspace*{50pt}}} \hspace*{8pt} \underset{\substack{i\text{'s observation} \\[1.5pt] \text{history}}}{\underbrace{\hspace*{48pt}}} \hspace*{8pt} \underset{\substack{\text{Whole other} \\[1.25pt] \text{system's state} \\ \text{history}}}{\underbrace{\hspace*{42pt}}} \hspace*{0pt}& \nonumber
\end{align} 

\noindent Note that the hatted variables in the domain are labelled from $0$ to $k$,
while the naked variables extend to $\ell$.  This models the fact that the
primary and second systems may operate at different time scales.  Immediately,
one sees other variations that are possible, such as instances when
$\hat{\pi}\rindi$ uses only the last element ($x_{\ell}$) of the secondary
system's state. Or, when the primary robots may communicate, the
$(i)$ superscripts may be dropped when we consider the multi-robot system
globally.  For simplicity, we restrict our attention in this paper only to the
basic case.  In
what follows, the term \emph{robot policy} refers to a function of the form
in~\eqref{eq:full_policy}.

\subsection{Illusions}
\begin{definition}\label{def:illusion}
    \label{defn:illusion}
  For deterministic multi-robot transition systems 
    $S = (n, X, U, f, Y, h, x_0)$,
  and
    $\hat{S} = (\hat{n}, \hat{X}, \hat{U}, \hat{f}, \hat{Y}, \hat{h}, \hat{x}_0)$,
  and 
  integer $0 < m \le n$, we say that $\hat{S}$ is an
  \emph{$m$-illusion} of $S$ if  there exist
  \begin{enumerate}
    \item[\footnotesize(i)] robot policies $\hat\pi\rind1,\ldots,\hat\pi\rind{\hat{n}}$ in $\hat{S}$, 
    \item[\footnotesize(ii)] a strictly increasing function
      $ z: \Z\plus \to \Z\plus$,
    and
    \item [\footnotesize(iii)] an infinite series of functions
      $ \rho_k: \Z_{m} \to \Z_{\hat{n}} $,
  \end{enumerate}
  for any robot policies $\pi\rind1,\ldots,\pi\rindn$ in $S$,
  such that for all $k \ge 0$ and all $1 \le i \le m$, we have
      \[ h\rindi(x_k) = \hat{h}\rind{\rho_k(i)}\left( \hat{x}_{z(k)}\right). \label{eq:faithful} \tag{$\star$} \]    
  Further, if $\hat{S}$ is an $m$-illusion of $S$, then a tuple of robot
  policies, mapping functions, and a time scaling function
    $(\hat\pi, (\rho_1, \rho_2, \ldots), z)$
  that ratifies the definition of illusion is called a \emph{witness} to that
  illusion.
\end{definition}

\noindent The preceding definition \gobblesubs{is perhaps sufficiently complex to warrant some dissection.}{warrants some dissection.}
\begin{tightenumerate}
  \item We understand the system $S$ to be the secondary one, i.e., the one
  that we intend to emulate.  The system $\hat{S}$ is the physical system whose
  execution will be orchestrated to appear, in the perception of some of its
  robots, to operate in the same manner as $S$.%
  \footnote{Occasionally human illusionists opt for for a
  certain type of stereotypical headwear (\raisebox{-0.5mm}{\includegraphics[height=1em]{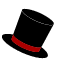}}).
  Likewise, our convention uses notation with hats
  (\raisebox{-0.3em}{$\:\hat{\hspace{1ex}}\:$}) to refer to systems whose robots
  are performing an illusion.  The parallel is unintentional but perhaps
  nonetheless a useful aid to understanding.} 

  \item The positive integer parameter $m$ is the number of robots in $S$ that
  are recipients of the illusion, whom we dub the participant robots.
  To simplify the notation, we will assume without loss of generality that
  the first $m$ robots in $S$, according to their indices, are the
  participants.
  (One might also expect for $m \le \bar{n}$ always, as it seems that 
  the number of participant robots cannot exceed the number of robots in the
  system; in fact, this need not be so, \foreign{cf.} Example~\ref{ex:loads_of_bumper}.)
  
  \item The $\hat{S}$ robot policies $\hat\pi\rindi$ described in
  condition~(i) govern the movements of the robots in that system.
  
  \item The function $z$ from condition (ii) establishes the relationship between the
  time scales of the two systems, so that $z(k)$ defines the physical time step
  in $\hat{S}$ corresponding to time step $k$ in $S$.

  \item The functions $\rho_k$ from condition (iii) indicate, for each time
  step $k$ of the execution in $S$, which robots of $\hat{S}$ play the roles of
  each of the participant robots in $S$.

\end{tightenumerate}
\gobblesubs{Finally, p}{P}ulling these elements together, the constraint marked~\eqref{eq:faithful}
requires, at each time step in $S$, that every participant robot is mapped, via
the $\rho$ function for that time step, to a robot in $\hat{S}$ that
experiences the same observation in that system as the mapped robot should
experience in $S$.
A few examples follow.

\begin{example}\label{ex:caravan2}
  Recall the autonomous caravan systems introduced in Example~\ref{ex:caravan}.
  For any such system $S = S_{n, \vmin, \vmin}$, we can form a $1$-illusion from any
  system of the form $\hat{S} = S_{3, \vhatmin, \vhatmax}$.
  This holds regardless of the number $n$ of robots in $S$ and of the
  range of actions $[\vhatmin, \vhatmax]$ available to each robot in $\hat{S}$.

  One way to construct such an illusion is to select a policy $\hat{\pi}$ in which robot~1 moves at a constant speed
  $(\vhatmin+\vhatmax)/2$.  The other two robots, knowing the desired
  observation $y\rind1_k = (a\rind1_k, b\rind1_k)$ from $S$, position
  themselves on opposite sides of robot~1, moving as fast as possible at each
  stage in $\hat{S}$ toward positions where 
    $\hat{x}_{\hat{k}}\rind1 - \hat{x}_{\hat{k}}\rind2 = b\rind1_k$
  and 
    $\hat{x}_{\hat{k}}\rind3 = \hat{x}_{\hat{k}}\rind1 = a\rind1_k$.
  To satisfy the remaining conditions of Definition~\ref{defn:illusion}, define
  $z$ to return the time when robots~2 and~3 in $\hat{S}$ have reached their target
  positions, and the sequence of mapping functions 
  $\rho_k: \{ 1 \} \to \{ 1, 2, 3 \}$ as a constant series of functions, 
  under which $1 \xmapsto{\rho_k} 1$
  for all $k$.  See Figure~\ref{fig:caravan2}.
  \vspace*{-8pt}
\end{example}

\begin{figure}
  \centering
  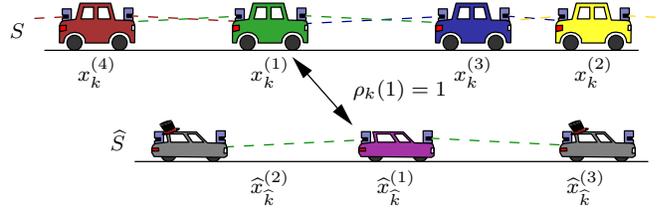
  \vspace*{-6pt}
  \caption{An illustration of Example~\ref{ex:caravan2}.  A system of three
  vehicles reproduces the observations expected in a system with potentially
  many more robots.\label{fig:caravan2}}
  \vspace*{-14pt}
\end{figure}

\begin{example}\label{ex:robotarium-ill}
  Recall the system $S\single$ introduced in Example~\ref{ex:pillars}.
  Suppose there exists an upper bound $m$ on the number of obstacles visible
  from\,---that is, within distance $r$ of---\,any position that the robot
  might reach.
  Then $S\disks$, from Example~\ref{ex:robotarium}, is a $1$-illusion for
  $S\single$, provided that it has at least $m+1$ robots, its workspace $W$
  is large enough to contain a circle of radius $r$, and the sensing range in
  $S\disks$ is no smaller than the sensing range in $S\single$.

  One way to achieve this illusion is to select robot~1 in $S\disks$ to act
  as the recipient of the observations as required by \eqref{eq:faithful}.
  This robot remains motionless at the center of the physical workspace $W$.
  At each stage $k$ in $S\single$, the desired observation $y\rind1_k$ is a
  list of positions at which robot~1 should perceive obstacles. 
  We choose a policy $\hat\pi$ that directs the some of the remaining
  $\hat{n}-1$ robots to those positions relative to robot 1, and directs the
  remaining robots to positions beyond its sensing range.  See
  Figure~\ref{fig:rpi}.  Many different policies, with varying degrees of time
  efficiency, can achieve this.
\end{example}

\begin{figure}
  \centering
  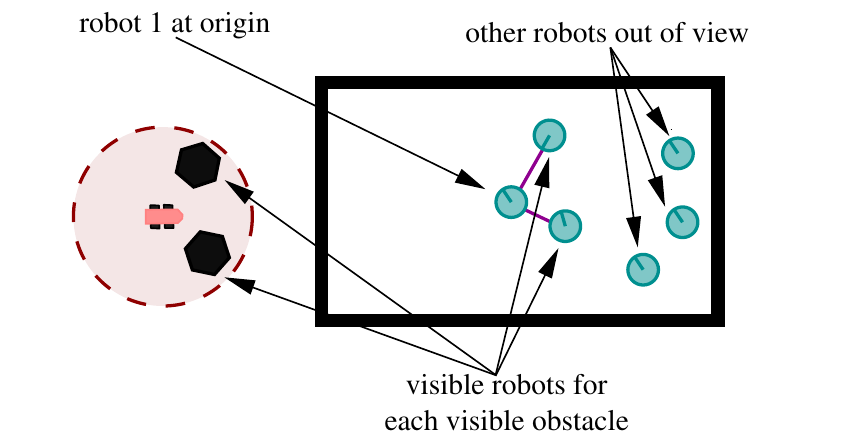
  \vspace*{-4pt}
  \caption{A $1$-illusion of $S\single$ using $S\disks$.
  \label{fig:rpi}}
  \vspace*{-16pt}
\end{figure}


\medskip
Next, we consider the execution time in the primary system as a resource
cost in which we are interested.

\begin{definition}
\label{defn:slowdown}
  If $\hat{S}$ is an $m$-illusion of $S$ with witness 
    $(\hat\pi, (\rho_1, \rho_2, \ldots), z)$,
  then the illusion is an $(m,\tau)$-illusion if the sequence
  \vspace*{-6pt}
    $$z(2)-z(1), z(3)-z(2), z(4)-z(3), \ldots  \vspace*{-4pt}$$
  is bounded above by $\tau$. 
  The constant $\tau$, which we can take to be an integer owing to the
  definition of $z$, is called the \emph{slowdown} of the illusion.
\end{definition}
In broad terms, we may then consider $\frac{1}{\tau}$, the inverse slowdown,
to be time efficiency of an illusion.

\begin{example}\label{ex:caravan3}
  Recall Example~\ref{ex:caravan2}.  That illusion has slowdown
    $\left\lceil 2(\vmax - \vmin)/(\vhatmax - \vhatmin) \right\rceil$.
\end{example}

\section{Basic Properties of Illusions}
Definition~\ref{def:illusion} provides a foundation for understanding the
notion of one system presenting an illusion of another.  Next, we present some
results that follow from that definition.
As an initial sanity check, we show that a system does indeed present a
faithful and efficient illusion of itself.

\begin{theorem}[identity]
A deterministic multi-robot transition system
    $S = (n, X, U, f, Y, h, x_0)$ is an
    ($n$,$1$)-illusion of $S$.
\end{theorem}
\begin{proof}
We observe that, if $z$ and $\rho_k$ are taken as 
identity functions, then \eqref{eq:faithful} holds when
$\hat\pi\rind{i} = \pi\rind{i}$. 
\end{proof}
\smallskip

Considering the preceding theorem, one might wonder whether a stronger
statement ought to be made, to the effect that every $S$ can provide an
$m$-illusion of itself for any $m < n$.
That statement is absent because it is false.  Supposing $m + p = n$ with $p >
0$, then there are $p$ robots that may show up under $h$.  Additional
properties of $h$ are needed to ensure that the $p$ robots can be made
invisible.

With additional assumptions on the dynamics of $S$, i.e., if the system
can be made to either loiter or affect state changes more slowly, then
an ($n$,$j$)-illusion with $j>1$ is also possible.
\medskip

\gobblesubs{A rather more interesting consideration is}{Rather more
interesting is} the nesting of systems: one system presenting an illusion to
another, that it itself presenting an illusion to a third.

\begin{theorem}[composition]
\label{thm:comp}
  If $\doublehat{\mbox{$S$}} = (\doublehat{n}, \doublehat{\mbox{$X$}}, \doublehat{\mbox{$U$}}, \doublehat{\mbox{$f$}}, \doublehat{Y}, \doublehat{\mbox{$h$}}, \doublehat{\mbox{$x$}}_0)$,
  is an $(\hat{n},\hat{\tau})$-illusion of 
    $\hat{S} = (\hat{n}, \hat{X}, \hat{U}, \hat{f}, \hat{Y}, \hat{h}, \hat{x}_0)$, and
$\hat{S}$
  is an $(m,\tau)$-illusion of 
      $S = (n, X, U, f, Y, h, x_0)$,
  then
  $\doublehat{\mbox{$S$}}$ is an $(m,(\tau\cdot\hat{\tau}))$-illusion of $S$.
\end{theorem}
\ifdefined\arxiv
  \begin{proof}
  Assume $(\hat\pi, (\rho_1, \rho_2, \ldots), z)$ is a witness for 
      $\hat{S}$'s illusion for $S$, and 
    $(\doublehat\pi, (\hat{\rho}_1, \hat{\rho}_2, \ldots), \hat{z})$ is a 
      witness for 
      $\doublehat{\mbox{$S$}}$'s  
  illusion for $\hat{S}$.
  To show $\doublehat{\mbox{$S$}}$ to be an $m$-illusion of $S$:
    \begin{enumerate}
      \item[\footnotesize(i)] take policies $\doublehat\pi$ in $\doublehat{\mbox{$S$}}$ because, since they must suffice for any 
  policies in $\hat{S}$, they must suffice for $\hat\pi$ in particular;
      \item[\footnotesize(ii)] take the function $\doublehat{\mbox{$z$}} = \hat{z} \comp z$
          is an increasing function, from $\Z\plus$ to $\Z\plus$, being the composition of two such functions;
      and
      \item [\footnotesize(iii)] the infinite series of functions
        $ \doublehat{\mbox{$\rho$}}_k: \Z_{m} \to \Z_{\footnotesize\mbox{$\doublehat{n}$}}$, with 
          $\doublehat{\mbox{$\rho$}}_k = \hat{\rho}_{z(k)}\comp \rho_k$. 
    \end{enumerate}
      The definition of $\doublehat{\mbox{$z$}}$ means that, for all $k$,
  $\doublehat{\mbox{$z$}}(k+1) -\doublehat{\mbox{$z$}}(k) = 
  \hat{z}(z(k+1)) - \hat{z}(z(k)) = 
  \hat{z}(z(k) + \tau - \Delta_k) - \hat{z}(z(k))$, for some $\Delta_k \geq 0$.
  Since $\hat{z}(k)$ is increasing, 
  $\hat{z}(z(k) + \tau - \Delta_k) - \hat{z}(z(k)) \leq
  \hat{z}(z(k) + \tau) - \hat{z}(z(k)).$
  But, since $\tau \in \Z\plus$, telescope the series:
  $\doublehat{\mbox{$z$}}(k+1) -\doublehat{\mbox{$z$}}(k) \leq
  \hat{z}(z(k) + \tau) - \hat{z}(z(k)) = 
  [\hat{z}(z(k) + \tau) 
  - \hat{z}(z(k) + \tau-1)]
  + [\hat{z}(z(k) + \tau-1)  
  - \hat{z}(z(k) + \tau-2) ]
  +\cdots 
  + [\hat{z}(z(k) + 1) - \hat{z}(z(k))] \leq \tau' + \cdots + \tau' = \tau \cdot \tau'$. 
  \end{proof}
  %
  %
  %
  %
  %
  %
  %
  \bigskip
\else
  \begin{proofdeferred}
    Omitted due to space limitations.  See~\cite{this_arxiv}.
  \end{proofdeferred}
\fi

Note that, in (iii), function composition requires that 
$\doublehat{\mbox{$S$}}$ be an $\hat{n}$-illusion of $\hat{S}$ in order for \gobblesubs{the domain
and ranges to match up}{the types to agree}. If $\doublehat{\mbox{$S$}}$ were only an $\hat{m}$-illusion of
$\hat{S}$ with $\hat{m} < \hat{n}$, then
the $\hat{n}-\hat{m}$ extra robots are needed to create an illusion for $S$.
This arises because we do not talk of some subset of robots in one system
sufficing to provide an illusion of another system, since all the primary
robots need to participate to ensure the illusion succeeds, even
if participating constitutes moving to ensure they're 
unobserved, ruining the illusion otherwise\gobble{ by giving the game away}.
\gobble{Again, this is the
requirement that some robots be invisible again.}

%
%
%
%
%

Illusions hold up to the set of observations made in the secondary system. 
One might expect that $Y \subseteq \hat{Y}$ but, in fact, $Y$ may be
larger or smaller, though the pair cannot be disjoint. It is not the range 
but the image which matters:
\begin{definition}\label{defn:percepti_foot}
The \emph{perceptual occurrence} of deterministic system $S = (n, X, U, f, Y, h, x_0)$,
is
the subset of $Y$, denoted $\seen{Y}$, that is produced under $h$ via 
states reachable by some robot policies $\pi\rind1,\ldots,\pi\rindn$.
\end{definition}

In Definition~\ref{defn:illusion}, requirement \eqref{eq:faithful} implies that
if $(\hat{n}, \hat{X}, \hat{U}, \hat{f}, \hat{Y}, \hat{h}, \hat{x}_0)$ is an
illusion of $(n, X, U, f, Y, h, x_0)$, then
$\seen{Y} \subseteq \seen{\hat{Y}}$.

\gobblesubs{In light of this,}{Now} we might inquire as to the implications for illusions under
alteration of the robots' sensors. We model potential degradation, or
preimage coarsening, of sensors via a function in the observation space, where
non-injective transformations will conflate things that were distinguishable
formerly.

\begin{theorem}[coarser observations]
\label{thm:coarser_obs}
  If $(\hat{n}, \hat{X}, \hat{U}, \hat{f}, \hat{Y}, \hat{h}, \hat{x}_0)$ 
  is an illusion of $(n, X, U, f, Y, h, x_0)$,
  then, for any function
  $\kappa: Y \cup \hat{Y} \to Z$, 
  we have that
  $(\hat{n}, \hat{X}, \hat{U}, \hat{f}, Z, \kappa\comp\hat{h}, \hat{x}_0)$ 
  is an illusion of $(n, X, U, f, Z, \kappa\comp h, x_0)$.\footnotemark
\end{theorem}
\begin{proof}
The original witness ratifies the new illusion, since
\vspace*{-4pt}
{\small
\[ h\rindi(x_k) = \hat{h}\rind{\rho_k(i)}\left( \hat{x}_{z(k)}\right)
\implies 
\kappa\comp h\rindi(x_k) = \kappa\comp\hat{h}\rind{\rho_k(i)}\left( \hat{x}_{z(k)}\right), \vspace*{-4pt}\]}
in which the left equality needs to hold over $\seen{Y}$ only.
\end{proof}

\footnotetext{This theorem holds for a slightly broader, albeit more obscure, class of functions.
One may take the disjoint union as the domain, $\kappa: Y
\sqcup \hat{Y} \to Z$, so long as there is agreement on the function restrictions up to
perceptual occurrence in the secondary system, i.e.,
$\forall y \in \seen{Y}$, 
\vspace*{-1pt}\mbox{$\kappa|_{\seen{Y}\phantom{\seen{\hat{Y}}}}\hspace*{-14pt}(y) = \kappa|_{\seen{\hat{Y}}}(\hat{y})$}.
%
}

\smallskip

It may seem, intuitively, that if $S$'s sensors are weakened, then that should
only make illusionability more feasible.  But for an illusion to be
passable, the definition requires that it appear identical to $S$, which
thus prohibits the robot's sensors from operating with implausibly high
fidelity.  We note that, though beyond the scope of this work, if one may alter
the secondary robot system, then the story changes. One could apply
$\kappa(\cdot)$ computationally, degrading after the sensor's signals
\foreign{ex post facto}, by introducing a small software shim in the position
indicated with the~\textcolor{markercolor}{$(\ddagger)$} in
Figure~\ref{fig:phil}.  
\label{ref:sw_shim} 

\section{The limits of illusion}

Why have two definitions (Definition~\ref{defn:illusion}
and~\ref{defn:slowdown}) to separate $m$-illusions from ($m$,
$\tau$)-illusions?  The next result establishes that pairs of systems exist
where the primary system is sufficiently powerful to conjure an illusion of the
second, but the gap in relative speeds has no limit. 
Put another way,
for any execution in the one, the other can create a faithful
illusion\gobble{(satisfying Definition~\ref{defn:illusion})}, but no bound
exists on the illusion's slowdown (i.e., there is no finite $\tau$ such that it
is an $(m,\tau)$-illusion\gobble{, in Definition~\ref{defn:slowdown}}).  The
result is \gobblesubs{not merely for the given illusion, but}{that} it is
impossible for the primary system to present \emph{any} illusion of the
secondary system satisfying Definition~\ref{defn:slowdown}. 

\begin{theorem}[Illusions with no bounded $\tau$]
\label{thm:noboundedillusion}
There exist deterministic multi-robot transition systems
$S$ and $\hat{S}$ where the latter is an $m$-illusion of the former,
but for which no $\tau$ exists such that it is an $(m, \tau)$-illusion.
\end{theorem}
\vspace*{-4pt}
\begin{proofsansqed}
We give constructions for both $S$ and $\hat{S}$, 
then show 
that $\hat{S}$ is indeed a $1$-illusion of $S$ (Lemma~\ref{lem:is_illusion});
and also, that any desired bound placed on the slowdown will
be surpassed (Lemma~\ref{lem:long_long}).
\end{proofsansqed}

\begin{figure}[b]
\vspace*{-16pt}
\centering
\includegraphics[width=.9\linewidth]{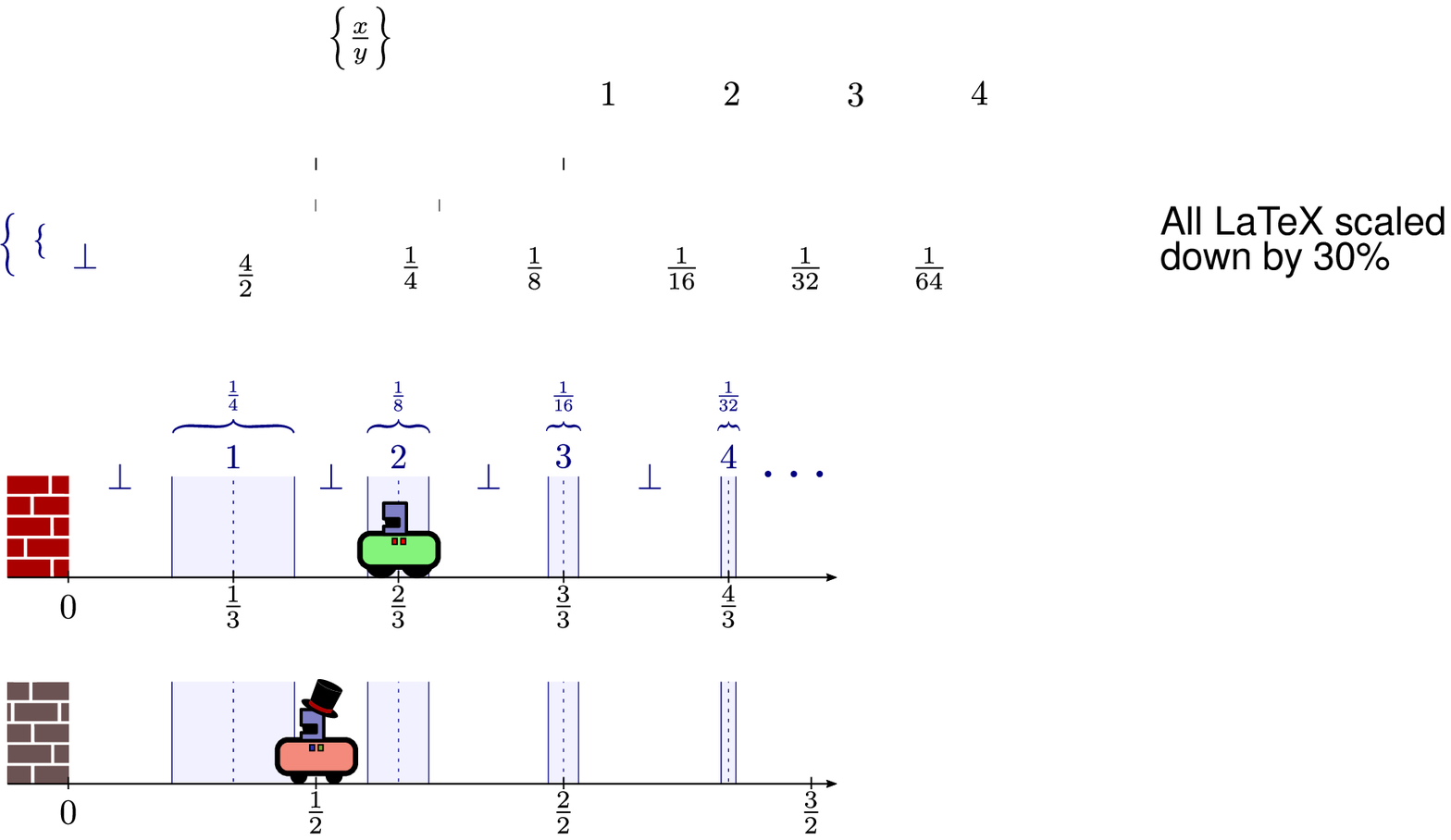}\vspace*{-10pt}
\caption{
A visual representation  of the 
two systems in
Constructions~\ref{const:noboundedillusion} and~\ref{const:noboundedillusion2}.
The green robot at the top of the figure is $S_{\text{thirds}}$,
the red robot below is $\hat{S}_{\text{binary}}$. Both measure the
environment with $h_{\text{sqz}}$, yielding a 
sensor whose preimage information is diagrammed in blue.
\label{fig:twoandthrees}
}
\end{figure}

\begin{construction}[$S_{\text{thirds}}$]\label{const:noboundedillusion}
We define the following deterministic multi-robot transition system
\vspace*{-2pt}
\begin{equation*}
      S_{\text{thirds}} = \left( 1, \R\plus,  \left\{-\tfrac{1}{3}, 0, \tfrac{1}{3}\right\}, f_{\text{add}}, \{\bot\}\cup\Z\plus, h_{\text{sqz}}, x_0=0\right),
\vspace*{-2pt}
\end{equation*}
where 
$f_{\text{add}}(x,u) =f_{\text{add}}\rind1(x,u)=x+u$, and, dubbed \textsl{squeeze},  
\vspace*{-2pt}
\begin{equation*}
      h_{\text{sqz}}(x)=\left\{\begin{array}{ll}
              q & \text{if } \exists q \in \Z\plus \text{ s.t. } - \frac{1}{4\cdot2^{q}} \leq x-\frac{q}{3}  \leq \frac{1}{4\cdot 2^{q}},\\
              &\vspace*{-8pt}\\
              \bot  & \text{otherwise.}
              \end{array}\right.
\vspace*{-2pt}
\end{equation*}
This describes a robot that lives on the positive $x$-axis and which moves along
in discrete steps, each with size $\frac{1}{3}$ units.
This is shown as the green robot in the top diagram in
Figure~\ref{fig:twoandthrees}.  The robot is equipped with a stylized range
sensor that measures a quantized distance to an obstacle at the origin (the
blue information in the diagram). The sensor's precision increases
(geometrically) with increasing $x$, with readings outside stripes of
increasing precision return a generic reading,~$\bot$.  (The sensor's behavior
here is essentially arbitrary for the construction\gobblesubs{: the sensor could,
in fact, give perfect readings for this region so long the outputs do not
collide in the observation space---thus, we've opted to use the generic symbol
to emphasize its insignificance.}{, the symbol emphasizes its insignificance.})
\end{construction}

\begin{construction}[$\hat{S}_{\text{binary}}$]\label{const:noboundedillusion2}
Next, consider deterministic multi-robot transition system
\vspace*{-2pt}
\begin{equation*}
      \hat{S}_{\text{binary}} = \left( 1, \R\plus,  \left\{\tfrac{1}{2^p}\mid p \in \Z\right\}, f_{\text{add}}, \{\bot\}\cup\Z\plus, h_{\text{sqz}}, \hat{x}_0=0\right),\vspace*{-2pt}
\end{equation*}
where $f_{\text{add}}$ and $h_{\text{sqz}}$ are as in
the preceding construction.

This robot also lives on the positive $x$-axis and moves in steps.
It has rather more options for its movement, it moves in either direction
with steps that are negative powers of two. This is shown as the red robot in
the bottom part of Figure~\ref{fig:twoandthrees}, where the arrows show `hops' of length
$-\frac{1}{4}, -\frac{1}{8}, \frac{1}{2}, \frac{1}{4}, \frac{1}{8},
\frac{1}{16}$, these being a sample of some actions available to the
robot.
\end{construction}

\begin{lemma}
$\hat{S}_{\text{binary}}$ is a \emph{$1$-illusion} of $S_{\text{thirds}}$.
\label{lem:is_illusion}
\end{lemma}
\ifdefined\arxiv
  \begin{proof}
  Function $\rho_k$ is determined because the systems have only one robot each. 
  Since the systems share the same state space and observation function, the
  approach is to have the robot in $\hat{S}_{\text{binary}}$ navigate
  toward the position that the robot in $S_{\text{thirds}}$ would appear in.
  Given the other system's previous state $x_\ell$ and the action it 
  wishes executed $u_\ell$, the hatted robot computes the target position
  $x_{\text{tgt}} = x_\ell + u_\ell$. It then compares this with its current
  position (computed from, $\hat{u}_{0}, \dots,\hat{u}_{k}$, integrating forward
  from $x_0$).  If the positions are equal, which can happen at integral
  positions, there is nothing that needs doing and the $z$ function causes 
  $S$ to continue.  Otherwise, the comparison indicates whether the movement will
  be in the positive $x$ direction, or the reverse---which involves selecting the
  appropriate sign.  Next, enumerate $1, \frac{1}{2}, \frac{1}{2^2},
  \frac{1}{2^3}, \frac{1}{2^4}, \dots$ until a step size is found that is
  sufficiently small to ensure the robot will not overshoot the target position.
  The hatted robot then takes this action. If the resulting position is still
  more than $\frac{1}{4\cdot 2^{3\cdot x_{\text{tgt}}}}$ units away from
  $x_{\text{tgt}}$, this last step is repeated: first computing the largest step
  size that doesn't jump over the target, then taking that step.  This process
  converges on $x_{\text{tgt}}$, and terminates when close enough.

  Generally, the hatted robot takes multiple steps to get into a position close
  enough to appear in the right region under $h_{\text{sqz}}$. These multiple
  steps are plateaus in $z$. 
  With more precision being needed the further the robots are from the origin,
  the number of steps in the plateaus will depend on the
  $x$ coordinate.  Nevertheless, for any position the $S_{\text{thirds}}$ robot
  wishes to occupy, there are a finite number of steps that the
  $\hat{S}_{\text{binary}}$ one needs to take as the target region is an
  interval with distinct endpoints and, therefore, contains some finite binary
  fraction.
  \end{proof}
\else
  \begin{proofdeferred}
    Omitted due to space limitations.  See~\cite{this_arxiv}.
  \end{proofdeferred}
\fi

\begin{lemma}
\label{lem:long_long}
For any $1$-illusion of $S_{\text{thirds}}$ by $\hat{S}_{\text{binary}}$,
and any finite $T$, the constant policy $u_k = \frac{1}{3}$  for
the robot in  
$S_{\text{thirds}}$ implies that some $0<N_T$ exists where
    $$T < \max_{k\in \{1,\dots,N_T\}}\left\{z(k+1)-z(k)\right\}.$$
\end{lemma}
\ifdefined\arxiv
  \begin{proof}
  %
  Suppose a $1$-illusion of $S_{\text{thirds}}$ by $\hat{S}_{\text{binary}}$
  is given. Now consider the constant policy $u_k = \frac{1}{3}$  
  with the robot in $S_{\text{thirds}}$ moving to the right.
  For every time $t$, the robot in $S_{\text{thirds}}$ wishes to have reached
  reached state $x_t = \frac{1}{3}t$. 
  Consider a time $t = 3h + 1$ for some $h\in\Z\plus$.
  At that point, 
  the robot in $\hat{S}_{\text{binary}}$ must be 
  in $h_{\text{sqz}}^{-1}(h+\frac{1}{3})$, that is, the preimage corresponding
  to the observation to be seen by the robot in $S_{\text{thirds}}$. But that
  means that 
  $|x_{3h+1} - \hat{x}_{z(3h + 1)}| \leq \frac{1}{2^{3+3h}}$, or
  $2^{3h}8|x_{3h+1} - \hat{x}_{z(3h + 1)}| \leq 1$.
  This means that, at time $3h+1$, if the states of robots in the respective
  systems are written in binary form, they will certainly agree up to the first
  $3h$ digits after the point.

  At time step $3h+1$, the binary representation of the state is $x_{3h+1} =
  \cdots b_3 b_2 b_1 b_0.0101010101\cdots$, where the bits to the left of the
  point represent $3h$.  At the next time step, the state is $x_{3h+2} = \cdots
  b_3 b_2 b_1 b_0.1010101010\cdots$.  At step $3h+1$, $\hat{x}_{3h+1}$ agrees on
  the first $3h$ digits. The robot in $\hat{S}_{\text{binary}}$ must move to a
  $\hat{x}_{3h+2}$ that will have to agree on at least the first $3h$
  digits---but those bits have all flipped. 
  The motion model of system $\hat{S}_{\text{binary}}$ permits it to 
  to add or subtract numbers that, when expressed in binary, comprise only a
  single $1$ bit to the right of the point. 
  This is the only way it permits state
  to change, no matter the mechanism employed by the illusion.
  Either addition or subtraction
  of such numbers can trigger an effect of altering a chain of bits through
  the carry mechanism (either a string of $1$s for addition, or a string of
  $0$s for subtraction). Because we start with $0$s and $1$s alternating
  in the first $3h$ digits, an amortized analysis shows that even those
  steps which seem to trigger long bit changes, must have been paid for before
  to set them up.  (See, for example, Section~17.1 of Cormen et al.~\cite{CLRS}, for details of this particular amortized analysis.)  The most efficient means to flip the bits of $\hat{x}$
  takes at least $\left\lfloor \frac{3h}{2} \right\rfloor$ steps.
  As time evolves, $h$ increases and the robot
  in $S_{\text{thirds}}$ moves steadily to the right, but the steps needed by the
  robot $\hat{S}_{\text{binary}}$ to maintain a plausible illusion between times
  $3h+1$ and $3h+2$ is not constant but costs at least 
  $\left\lfloor \frac{3h}{2} \right\rfloor$, i.e.,
  $\frac{3h}{2}-1 < \left\lfloor \frac{3h}{2} \right\rfloor \leq z(3h+2) - z(3h+1)$.  Taking $N_T = \frac{2}{3} T + 3$ thus ensures that the
  condition is met.
  \end{proof}
\else
  \begin{proofdeferred}
    Omitted due to space limitations.  See~\cite{this_arxiv}.
  \end{proofdeferred}
\fi

\bigskip
\begin{proof}
The preceding two Lemmas prove Theorem~\ref{thm:noboundedillusion}.
\end{proof}

\section{Physical demonstration in the Robotarium}
As a proof-of-concept, we implemented the illusion described in
Example~\ref{ex:robotarium-ill} both in simulation and on a physical robot
testbed.  Simulations were conducted using an implementation in Python;
physical experiments were conducted in the
Robotarium~\cite{pickem2017robotarium}.
Figure~\ref{fig:still} shows a snapshot of the execution.  Refer also to the
supplemental video.

\begin{figure}
  \centering
  \includegraphics[width=\columnwidth]{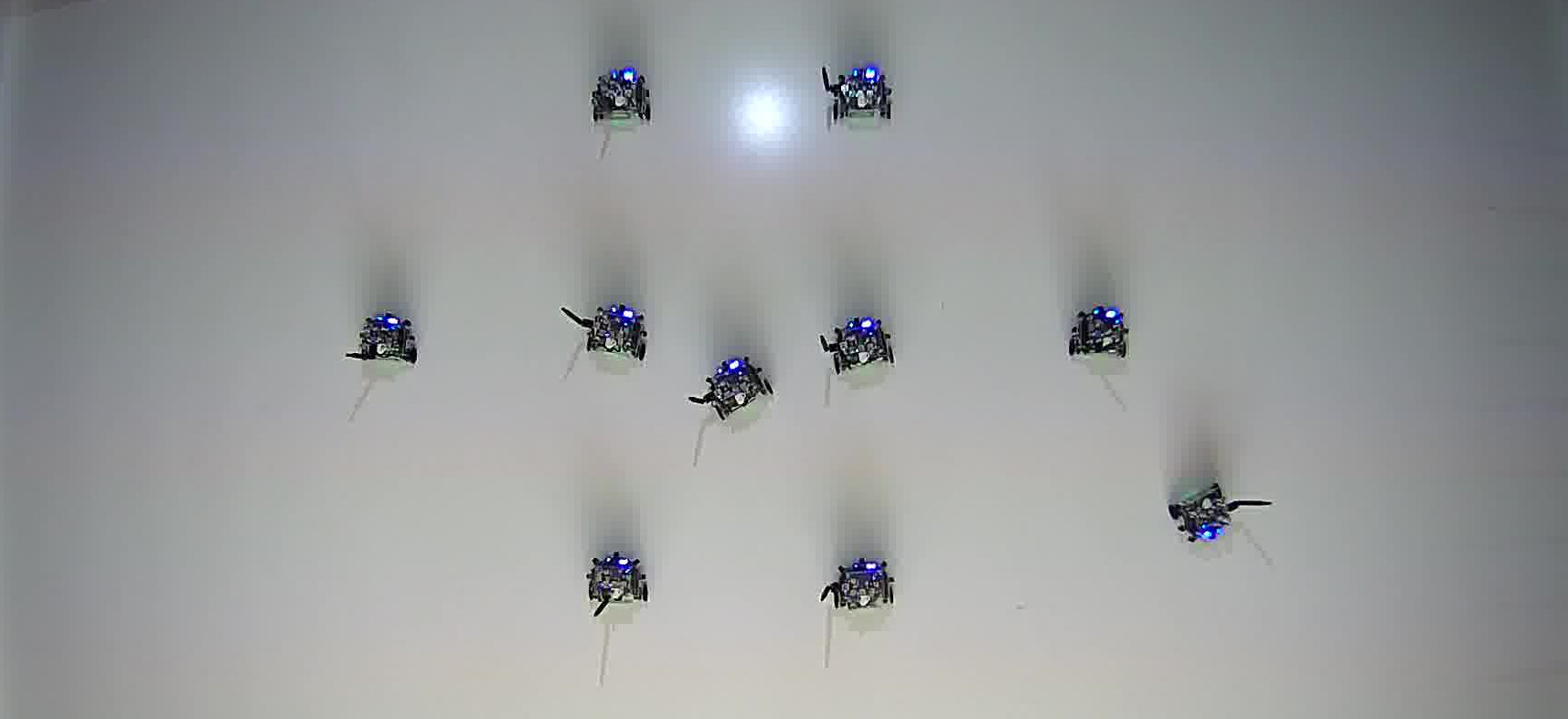}\vspace*{-2pt}
  \caption{A collection of 10 robots performing the illusion of
  Example~\ref{ex:robotarium-ill}.  Robot~1 (center) remains motionless.
  Eight other robots play the roles of eight nearby obstacles.  A tenth robot
  (far right) remains out of view.} \label{fig:still}
\vspace*{-10pt}
\end{figure}

Note that Example~\ref{ex:robotarium-ill} calls for the complicit robots to
assume certain positions, but does not prescribe which robots should take
which roles.  We implemented three distinct strategies:
\begin{enumerate}
  \item[\footnotesize(i)] A \textbf{na\"{\i}ve} matching strategy, in which
  robots are assigned to \gobble{obstacle} roles from left to right, in order of their
  indices.
  \item[\footnotesize(ii)] The \textbf{Hungarian}
  algorithm~\cite{kuhn55,munkres57} for optimal task assignment, wherein some
  robots are assigned to obstacle roles and the remaining robots travel to
  the nearest location outside of the sensor range.  The matching is selected
  to minimize the total travel time. \gobble{needed to reach the desired positions.}
  \item[\footnotesize(iii)] An enhancement of the \textbf{Hungarian} strategy
  with a \textbf{heuristic} that directs the offstage robots to the locations
  of the nearest obstacles that are not yet visible.
\end{enumerate}
One might expect, in this context, that the time efficiency of the illusion
might be impacted both by the number of robots employed in the physical system
and by the policy used in that system to carry out the illusion.
To test this hypothesis, we performed a series of simulations of the policies
described above.
We conducted 10 trials, each using a distinct randomly-generated path for the
robot in $S$.  For each, we executed each of the three illusions
described above and measured the amount of real time in $\hat{S}$ needed to
execute the policy in $S$.

Several notable trends appear in the results, which are shown in
Figure~\ref{fig:big_experiment}.  Most plainly, the relative efficiency
between the three algorithms matches what one might expect: Better use of
more information leads to a more time-efficient illusion.  For the two methods
based on Hungarian matching, opposite trends appear as the number of robots
increases: the basic Hungarian approach loses efficiency as robots are added,
presumably due to interference from avoiding collisions between the robots.  In
contrast, the heursitic that positions robots near locations where new
obstacles are likely to appear in the future is better able to take advantage
of additional robots waiting `in the wings' to take on roles when needed,
leading to improvements in efficiency as the number of robots increases.

\begin{figure}
  \centering
  \scalebox{0.5}{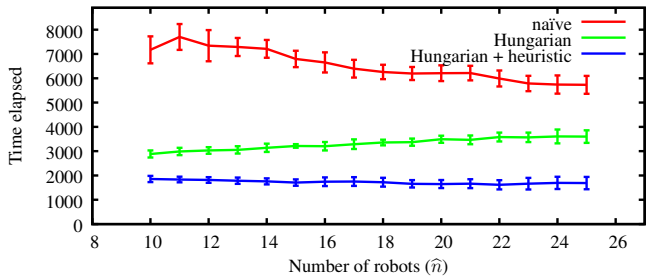}
  \vspace*{-4pt}
  \caption{Simulation results showing the impact of the number of robots and
  the policy on time efficiency.}
  \label{fig:big_experiment}
  \vspace*{-8pt}
\end{figure}

\section{Outlook and conclusions: So now what?}
\label{sec:conclusion}

There can be an immense variety of very different means to realize the same
illusion. The single lesson that emerges most clearly from our demonstration
implementation\,---both the more thorough simulation trials and the physical
instance on the robotarium, where a time cut-off is imposed---\,is that
distinct approaches may have time efficiencies that differ considerably.  Even
the more efficient curve in Figure~\ref{fig:big_experiment} has a slowdown
factor of about~9, which is likely an impediment when producing an illusion of
robots that one has direct access to.  But consider an illusion for the system
in Example~\ref{ex:pillars} where the field of obstacles is unbounded: it
simply can't be achieved physically.  Moreover, if the $1$-illusion has both
the participant and the obstacles moving, it is possible to present an illusion
for a robot that is faster than any we own.  Judging the value of the idea by
an early implementation is probably unwise, though an order of magnitude gap is
not always fatal (compare, for instance, software simulation of {\sc vlsi}
circuits versus hardware).

Several research directions remain, some more pressing than others. We lead with
those we deem most important:
\begin{tightitemize2}
\label{sec:poss_ideas}
\item Extensions to address uncertainty and non-determinism would be most
valuable. Some basic questions are still unresolved: is $\hat{S}$
permitted structure interactions so that only \emph{some} of the 
outcomes arise, or must all be possible? If a probabilistic perspective is
adopted, do illusions have to present events with representative statistics?
\item A richer theory of efficient illusions is needed to empower reasoning
about resource trade-offs with regards to illusions. For instance,
Theorem~\ref{thm:coarser_obs} says nothing about efficiency.\gobble{, but in light of
that result, c} Can sensor preimage coarsening reduce the slowdown?  How can one
\gobble{better} understand the trade-offs between actuator capabilities and the
illusions that can be produced?
\gobble{\item One weaker condition for the ability to contrive an $m$-illusion would
permit some prior knowledge of the control policies used by robot system $S$.
Though perhaps $\hat{S}$ is too weak in general, 
it might have sufficient power for certain instances of interest.
}
\item Another weaker condition for the ability to produce illusions
might impose a notion of distance (or at least some topology) on the 
observation space $Y$ so that almost-illusions or
probably-approximately-illusions might be formalized. If $S$ cannot be produced
via $\hat{S}$, we might settle for less: the $S'$ that is `closest' to $S$.
\item Though this paper has not addressed it head-on, some modeling considerations
can be subtle. For example, whether velocity, or other aspects tied to physical
time, are part of the state space, $X$, or not is tricky.  This is, at least
somewhat, partly anticipated in~\cite{lavalle07time}. 
\item How to best model a system $\hat{S}$ producing two illusions
simultaneously? This would allow one to develop a notion of multiprogramming
for timeshared physical robot resources, like the Robotarium. Scheduling need
not occur at the level of whole experiments, instead robots are more like
virtual memory, where more fine-grained concurrency is possible.
\item Can one consider, systematically, what is gained by having greater
influence over the robot? The software shim mentioned at the very end of
Section~\ref{ref:sw_shim} is but one instance. Another alluring possibility, if
we can permit per-robot delays in receipt of sensor signals, or even caching of
them, is to weaken the requirement of temporal linearity. Doing so could lead
to a sort of `out-of-order emulation' for robots and a concomitant acceleration
of the execution.

\end{tightitemize2}



\bibliographystyle{plainnat}
\bibliography{bibrss}

\end{document}